\documentclass[10pt]{article}
\usepackage{amsmath,amsthm,amsfonts,amssymb,latexsym,graphicx,bbold,epigraph}
\usepackage[pdfpagemode=UseNone,pdfstartview=FitH]{hyperref}
\usepackage[T2A]{fontenc}
\usepackage[utf8]{inputenc}

\allowdisplaybreaks[4]

\newcommand{\Extra}[1]{}

\newtheorem{corollary}{Corollary}
\newtheorem{theorem}{Theorem}
\theoremstyle{definition}
\newtheorem*{remark}{Remark}

\newcommand*{\CTIV}{Vovk:arXiv0904}

\newcommand{\Logik}{Popper:1934}

\newcommand{\st}{\mid}

\DeclareMathOperator{\PS}{PS}
\DeclareMathOperator{\TT}{TT}
\DeclareMathOperator{\EE}{EE}

\newcommand{\CCC}{\mathbf{C}}
\newcommand{\KKK}{\mathbf{K}}
\newcommand{\MMM}{\mathbf{M}}
\newcommand{\DDD}{\mathbf{D}}

\DeclareMathOperator{\dom}{dom}

\title{Universal probability-free prediction}
\author{Vladimir Vovk and Dusko Pavlovic}

\begin{document}
\maketitle

\begin{abstract}
  We construct universal prediction systems
  in the spirit of Popper's falsifiability and Kolmogorov complexity and randomness.
  These prediction systems do not depend on any statistical assumptions
  (but under the IID assumption they dominate, to within the usual accuracy, conformal prediction).
  Our constructions give rise to a theory of algorithmic complexity and randomness of time
  containing analogues of several notions and results
  of the classical theory of Kolmogorov complexity and randomness.

  \bigskip

  \noindent
  The conference version of this paper has been published in the Proceedings of COPA 2016.
  The journal version is to appear
  in the Special Issue of the \emph{Annals of Mathematics and Artificial Intelligence}
  devoted to COPA 2016.
  The version at \href{http://alrw.net}{http://alrw.net} (Working Paper 14) is updated most often.
\end{abstract}

\setlength{\epigraphwidth}{0.63\textwidth}

\section{Introduction}

In this paper we consider the problem of predicting the labels, assumed to be binary, of a sequence of objects.
This is an online version of the standard problem of binary classification.
Namely, we will be interested in infinite sequences of observations
\[
  \omega
  =
  (z_1,z_2,\ldots)
  =
  ((x_1,y_1),(x_2,y_2),\ldots)
  \in
  (\mathbf{X}\times\mathbb{2})^{\infty}
\]
(also called \emph{infinite data sequences}),
where $\mathbf{X}$ is an \emph{object space} and $\mathbb{2}:=\{0,1\}$ is the \emph{label space}.
For simplicity, we will assume that the object space $\mathbf{X}$ is a given finite set of, say, binary strings
(the intuition being that finite objects can always be encoded as binary strings).
The elements 1 and 0 of the label space are often interpreted as ``true'' and ``false''.

Finite sequences $\sigma\in(\mathbf{X}\times\mathbb{2})^*$ of observations
will be called \emph{finite data sequences}.
If $\sigma_1,\sigma_2$ are two finite data sequences,
their concatenation will be denoted $(\sigma_1,\sigma_2)$;
$\sigma_2$ is also allowed to be an element of $\mathbf{X}\times\mathbb{2}$.
A standard partial order on $(\mathbf{X}\times\mathbb{2})^*$ is defined as follows:
$\sigma_1\sqsubseteq\sigma_2$ means that $\sigma_1$ is a prefix of $\sigma_2$;
$\sigma_1\sqsubset\sigma_2$ means that $\sigma_1\sqsubseteq\sigma_2$ and $\sigma_1\ne\sigma_2$.
The smallest element in this order (the empty data sequence) is denoted $\Box$.
We say that finite data sequences $\sigma_1$ and $\sigma_2$ are \emph{comparable}
if $\sigma_1\sqsubseteq\sigma_2$ or $\sigma_2\sqsubseteq\sigma_1$.

We use the notation $\mathbb{N}:=\{1,2,\ldots\}$ for the set of positive integers
and $\mathbb{N}_0:=\{0,1,2,\ldots\}$ for the set of nonnegative integers.
The \emph{length} of a finite data sequence $\sigma$ is the number $l\in\mathbb{N}_0$
such that $\sigma\in(\mathbf{X}\times\mathbb{2})^l$.
If $\omega\in(\mathbf{X}\times\mathbb{2})^{\infty}$ and $l\in\mathbb{N}_0$,
$\omega^l\in(\mathbf{X}\times\mathbb{2})^l$ is the prefix of $\omega$ of length $l$.

We will also use the notation $\sigma_1\sqsubseteq\sigma_2$ and $\left|\sigma\right|$
for finite binary sequences $\sigma_1,\sigma_2,\sigma\in\mathbb{2}^*$.

A \emph{situation} is a concatenation $(\sigma,x)\in(\mathbf{X}\times\mathbb{2})^*\times\mathbf{X}$
of a finite data sequence $\sigma$ and an object $x$;
our task in the situation $(\sigma,x)$
is to be able to predict the label of the new object $x$ given the sequence $\sigma$ of labelled objects.
Given a situation $s=(\sigma,x)$ and a label $y\in\mathbb{2}$,
we let $(s,y)$ stand for the finite data sequence $(\sigma,(x,y))$,
which is the concatenation of $s$ and $y$.

Our notation for binary logarithm will be $\log$.

\subsection*{The contents of this paper}

This paper is, to some degree, a result of our attempts to understand the philosophical problem of prediction.
It has two components, philosophical and mathematical,
and the latter is more or less independent of the former.
The main goal of the remainder of this section is to provide a road map
for our mathematical readers who do not share our philosophy of science,
so that the latter does not get in their way.
The four key mathematical concepts introduced in this paper are:
\begin{itemize}
\item
  universal prediction system (Sections~\ref{sec:laws}--\ref{sec:complexity}),
\item
  time complexity (Sections~\ref{sec:complexity}--\ref{sec:time-complexity}),
\item
  \emph{a priori} time semimeasure (Section~\ref{sec:semimeasure}),
\item
  time randomness (Section~\ref{sec:randomness}).
\end{itemize}
In the remaining sections we will explore (rather superficially) various connections
between these key concepts.

We start from a toy formalization of the philosophical notion of a law of nature in Section~\ref{sec:laws}
and the most basic way of using laws of nature for prediction in Section~\ref{sec:strong}.
The notion of a strong prediction system introduced in Section~\ref{sec:strong}
has only philosophical interest in this paper,
and this section can be safely skipped by our mathematical readers.
The notion of a weak prediction system introduced in the following Section~\ref{sec:weak}
is more convenient from the mathematical point of view
since there exists a universal weak prediction system, as shown in Section~\ref{sec:universal}.
Section~\ref{sec:complexity} introduces the notion of complexity for weak prediction systems
and uses it to strengthen the property of universality of the universal prediction system.

The reader who is not interested in prediction
can start from the mathematical notion of laws of nature in Section~\ref{sec:laws}
and the definition of their complexity in the second part of Section~\ref{sec:complexity},
which will prepare her to reading Section~\ref{sec:time-complexity}
about time complexity (apart from the theorem describing connections with universal prediction).

The definition of the \emph{a priori} time semimeasure in Section~\ref{sec:semimeasure}
is self-contained and does not depend on the previous sections.
This section contains simple connections between the \emph{a priori} time semimeasure
and time complexities.

Section~\ref{sec:randomness} devoted to time randomness is the nexus of this paper.
Time randomness is defined in terms of time complexity
(another natural definition would be in terms of \emph{a priori} time semimeasure)
and serves as the basis for prediction under conditions of considerable noise
(including connections with conformal prediction).

The last two sections, \ref{sec:conformal} and \ref{sec:Kolmogorov},
explain connections of the key mathematical concepts of this paper
with the theories of conformal prediction and Kolmogorov complexity,
respectively.

\section{Laws of nature as prediction systems}
\label{sec:laws}

\epigraph{Not for nothing do we call the laws of nature ``laws'':
  the more they prohibit, the more they say.}%
{\textit{The Logic of Scientific Discovery}\\ \textsc{Karl Popper}}\nocite{\Logik}

\noindent
According to Popper's \cite{\Logik} view of the philosophy of science,
scientific laws of nature should be falsifiable:
some finite sequences of observations should disagree with such a law,
and if so, we should be able to detect the disagreement.
(Popper often preferred to talk about scientific theories or statements instead of laws of nature.
He did not discuss the computational details of detecting disagreement;
for him, it was just something that we see straight away.)
The empirical content of a law of nature is the set of its potential falsifiers
(\cite{\Logik}, Sections 31 and 35).
We start from formalizing this notion in our toy setting,
interpreting the requirement that we should be able to detect falsification
as that we should be able to detect it eventually.

Formally, we define a \emph{law of nature} $L$ to be a recursively enumerable prefix-free subset of $(\mathbf{X}\times\mathbb{2})^*$
(where \emph{prefix-free} means that $\sigma_1\notin L$ whenever $\sigma_2\in L$ and $\sigma_1\sqsubset\sigma_2$).
Intuitively, these are the potential falsifiers,
i.e., sequences of observations prohibited by the law of nature.
The requirement of being recursively enumerable is implicit in the notion of a falsifier,
and the requirement of being prefix-free reflects the fact
that extensions of prohibited sequences of observations are automatically prohibited
and there is no need to mention them in the definition
(see, however, Remark~\ref{rem:Shen-1} below).
It is convenient to allow the vacuous law of nature $\emptyset$.

A law of nature $L$ gives rise to a prediction system:
in a situation $s=(\sigma,x)$ it predicts that the label $y\in\mathbb{2}$ of the new object $x$
will be an element of
\begin{equation}\label{eq:law}
  \Pi_L(s)
  :=
  \left\{
    y\in\mathbb{2}
    \st
    (s,y)\notin L
  \right\}.
\end{equation}
There are three possibilities in each situation $s$:
\begin{itemize}
\item
  The law of nature makes a prediction, either 0 or 1, in situation $s$
  when the prediction set \eqref{eq:law} is of size 1,
  $\left|\Pi_L(s)\right|=1$.
\item
  The prediction set is empty, $\left|\Pi_L(s)\right|=0$,
  which means that the law of nature is about to be falsified
  (and we can even say that it has been falsified already).
\item
  The law of nature refrains from making a prediction
  when $\left|\Pi_L(s)\right|=2$.
  This can happen in two cases:
  \begin{itemize}
  \item
    the law of nature was falsified in past:
    $\sigma'\in L$ for some $\sigma'\sqsubseteq\sigma$;
  \item
    the law of nature has not been falsified as yet
    and allows $(\sigma,(x,y))$ for both $y=0$ and $y=1$.
  \end{itemize}
\end{itemize}

\begin{remark}\label{rem:stopping-time}
  The counterpart of our notion of a law of nature in probability theory
  is that of a stopping time.
\end{remark}

\begin{remark}\label{rem:Shen-1}
  Our definition of a law of nature is the one that appears to us
  to lead to the simplest and richest theory,
  but there are several viable alternatives.
  Let us say that a subset $L$ of $(\mathbf{X}\times\mathbb{2})^*$ is an \emph{upset}
  if the conjunction of $\sigma_1\in L$ and $\sigma_1\sqsubseteq\sigma_2$ implies $\sigma_2\in L$.
  It is clear that the definition of laws of nature as upsets $L$ whose \emph{frontier}
  \[
    \{\sigma\in L\st\forall\sigma'\sqsubset\sigma:\sigma'\notin L\}
  \]
  is recursively enumerable is completely equivalent to ours;
  talking about frontiers of upsets rather than upsets is a matter of taste.
  We can both narrow down and widen up this definition in a natural way.
  The most restrictive definition discussed in this remark
  identifies a law of nature with a computable upset $L$ in $(\mathbf{X}\times\mathbb{2})^*$
  that is \emph{strongly co-continuable},
  in the sense of satisfying
  \[
    \forall\sigma\notin L \;
    \forall x\in\mathbf{X} \;
    \exists y\in\mathbb{2}:
    (\sigma,(x,y))\notin L
  \]
  (which is equivalent to the prediction set~\eqref{eq:law} never being empty
  unless the law has been falsified).
  A slightly less restrictive definition is to identify a law of nature
  with a computable upset that is \emph{weakly co-continuable},
  in the sense of satisfying
  \[
    \forall\sigma\notin L \;
    \exists (x,y)\in(\mathbf{X}\times\mathbb{2}):
    (\sigma,(x,y))\notin L.
  \]
  (Notice that a subset of $(\mathbf{X}\times\mathbb{2})^*$ is a weakly co-continuable upset
  if and only if it can be represented as the set of all finite data sequences
  all of whose infinite continuations are elements of a given open subset of $(\mathbf{X}\times\mathbb{2})^{\infty}$;
  this gives a natural one-to-one correspondence between the weakly co-continuable upsets
  and the open sets in $(\mathbf{X}\times\mathbb{2})^{\infty}$.)
  The main reasons we do not impose either of the conditions of co-continuability
  are that we opt for simpler definitions
  and that empty prediction sets are common in conformal prediction.
  A serious disadvantage of definitions involving the requirement of computability is that,
  in non-trivial cases, they do not allow constructing universal objects
  (such as universal prediction systems in Section~\ref{sec:universal} below).
  Dropping the two conditions of co-continuability
  and relaxing computability to recursive enumerability of the frontier,
  we obtain our definition.
  Further relaxing the requirement of computability,
  we can define a law of nature as a recursively enumerable upset.
  This is a natural definition that fits our intuition behind laws of nature
  even better than our official definition does;
  it also allows us to define universal predictions systems
  (as in Theorems~\ref{thm:basic} and~\ref{thm:precise} below).
  However, we do not know how to connect it with natural counterparts
  of the standard notions of Kolmogorov complexity, \emph{a priori} semimeasure, and randomness
  (as we do for our definition in Sections~\ref{sec:time-complexity}--\ref{sec:randomness}).
  And our official definition still covers the practically important case of computable upsets.
\end{remark}

\section{Strong prediction systems}
\label{sec:strong}

The notion of a law of nature is static;
experience tells us that laws of nature eventually fail and are replaced by other laws.
Popper represented his picture of this process by formulas (``evolutionary schemas'') similar to
\begin{equation}\label{eq:Popper}
  \PS_1 \to \TT_1 \to \EE_1 \to \PS_2 \to \cdots
\end{equation}
(introduced in his 1965 talk on which \cite{Popper:1979}, Chapter~6, is based
and also discussed in several other places in \cite{Popper:1979} and \cite{Popper:1999}).
In response to a problem situation $\PS$,
scientists create a tentative theory $\TT$ and then subject it to attempts at error elimination $\EE$,
whose success leads to a new problem situation $\PS$,
after which scientists come up with a new tentative theory $\TT$,
etc.
In our toy version of this process,
tentative theories are laws of nature,
problem situations are situations in which our current law of nature becomes falsified,
and there are no active attempts at error elimination
(so that error elimination simply consists in waiting until the current law of nature becomes falsified).

If $L$ and $L'$ are laws of nature,
we define $L\sqsubset L'$ to mean that for any $\sigma'\in L'$ there exists $\sigma\in L$
such that $\sigma\sqsubset\sigma'$.
To formalize the philosophical picture \eqref{eq:Popper},
we define a \emph{strong prediction system} $\mathcal{L}$
to be a nested sequence $L_1\sqsubset L_2\sqsubset\cdots$ of laws of nature $L_1,L_2,\ldots$
that are jointly recursively enumerable, in the sense of the set
$\{(\sigma,n)\in(\mathbf{X}\times\mathbb{2})^*\times\mathbb{N}\st\sigma\in L_n\}$
being recursively enumerable.

The interpretation of a strong prediction system $\mathcal{L}=(L_1,L_2,\ldots)$
is that $L_1$ is the initial law of nature used for predicting the labels of new objects
until it is falsified;
as soon as it is falsified we start looking for and then using for prediction
the following law of nature $L_2$ until it is falsified in its turn, etc.
Therefore, the prediction set in a situation $s=(\sigma,x)$ is natural to define as the set
\begin{equation}\label{eq:Gamma}
  \Pi_{\mathcal{L}}(s)
  :=
  \left\{
    y\in\mathbb{2}
    \st
    (s,y)\notin\cup_{n=1}^{\infty}L_n
  \right\}.
\end{equation}
As before, it is possible that $\Pi_{\mathcal{L}}(s)=\emptyset$.

Fix a situation $s=(\sigma,x)\in(\mathbf{X}\times\mathbb{2})^*\times\mathbf{X}$.
Let $n=n(s)$ be the largest integer such that $\sigma$ has a prefix in $L_n$.
It is possible that $n=0$ (when $s$ does not have such prefixes),
but if $n\ge1$, $s$ will also have prefixes in $L_{n-1},\ldots,L_1$,
by the definition of a strong prediction system.
Then $L_{n+1}$ will be the current law of nature;
all earlier laws, $L_n,L_{n-1},\ldots,L_1$, have been falsified.
The prediction \eqref{eq:Gamma} in situation $s$ is then interpreted
as the set of all labels $y$ that are not prohibited by the current law $L_{n+1}$.

In the spirit of the theory of Kolmogorov complexity,
we would like to have a universal prediction system.
However, we are not aware of any useful notion of a universal strong prediction system.
Therefore, in the next section we will introduce a wider notion of a prediction system
that does not have this disadvantage.

\section{Weak prediction systems}
\label{sec:weak}

A \emph{weak prediction system} $\mathcal{L}$ is defined to be a sequence
(not required to be nested in any sense)
$L_1,L_2,\ldots$ of laws of nature $L_n\subseteq(\mathbf{X}\times\mathbb{2})^*$
that are jointly recursively enumerable.

\begin{remark}
  Popper's evolutionary schema~\eqref{eq:Popper} was the simplest one that he considered;
  his more complicated ones, such as
  \begin{equation*}
    \PS_1
    \begin{aligned}
      \raisebox{-6pt}{$\nearrow$} & \TT_{\rm a} \to \EE_{\rm a} \to \PS_{\rm 2a} \to \cdots\\
      \to & \TT_{\rm b} \to \EE_{\rm b} \to \PS_{\rm 2b} \to \cdots\\
      \raisebox{6pt}{$\searrow$} & \TT_{\rm c} \to \EE_{\rm c} \to \PS_{\rm 2c} \to \cdots
    \end{aligned}
  \end{equation*}
  (cf.\ \cite{Popper:1979}, pp.~243 and~287),
  give rise to weak rather than strong prediction systems.
\end{remark}

In the rest of this paper we will omit ``weak'' in ``weak prediction system''.
The most basic way of using a prediction system $\mathcal{L}$
for making a prediction in situation $s=(\sigma,x)$ is as follows.
Decide on the maximum number $N$ of errors you are willing to make.
Ignore all $L_n$ apart from $L_1,\ldots,L_N$ in $\mathcal{L}$,
so that the prediction set in situation $s$ is
\[
  \Pi^N_{\mathcal{L}}(s)
  :=
  \left\{
    y\in\mathbb{2}
    \st
    \forall n\in\{1,\ldots,N\}:
    (s,y)\notin L_n
  \right\}.
\]
Notice that this way we are guaranteed to make at most $N$ mistakes:
making a mistake eliminates at least one law
in the list of unfalsified laws among $\{L_1,\ldots,L_N\}$.

Similarly to the theory of conformal prediction
(see, e.g., \cite{Vovk:2014Bala}),
another way of packaging $\mathcal{L}$'s prediction in situation $s$ is,
instead of choosing the threshold (or \emph{level}) $N$ in advance,
to allow the user to apply her own threshold:
in a situation $s$, for each $y\in\mathbb{2}$ report the attained level
\begin{equation}\label{eq:pi}
  \pi^s_{\mathcal{L}}(y)
  :=
  \min
  \left\{
    n\in\mathbb{N}
    \st
    (s,y)\in L_n
  \right\}
  \in
  \mathbb{N} \cup \{\infty\}
\end{equation}
(with $\min\emptyset:=\infty$).
The user whose threshold is $N$ will then consider $y\in\mathbb{2}$ with $\pi^s_{\mathcal{L}}(y)\le N$ as prohibited in $s$.
Notice that the function \eqref{eq:pi} is upper semicomputable (for a fixed $\mathcal{L}$).

The strength of a prediction system $\mathcal{L}=(L_1,L_2,\ldots)$ at level $N\in\mathbb{N}$
is determined by its \emph{$N$-part}
\begin{equation}\label{eq:N-part}
  \mathcal{L}_{\le N}
  :=
  \bigcup_{n=1}^N
  L_n.
\end{equation}
At level $N$,
the prediction system $\mathcal{L}$ prohibits $y\in\mathbb{2}$ as continuation of a situation $s$
if and only if $(s,y)\in\mathcal{L}_{\le N}$.

We will also use the limit
\begin{equation}\label{eq:limit}
  \mathcal{L}_{<\infty}
  :=
  \bigcup_{n=1}^{\infty}
  L_n
\end{equation}
of \eqref{eq:N-part}.
Notice that $\mathcal{L}_{<\infty}$ uniquely determines $\mathcal{L}$
if $\mathcal{L}$ is a strong prediction system,
but the analogous statement for weak prediction systems is false.

\begin{remark}
  In motivating our definitions we have referred to views expressed in Karl Popper's writings.
  Similar views have been held by many other philosophers.
  Popper himself\Extra{\ (\cite{Popper:1979}, p.~297, below the displayed equation)}
  regarded his evolutionary schemas as improvements and rationalizations
  of the Hegelian dialectical schema.
  Charles Peirce's views were particularly close to Popper's.
  He was as emphatic as Popper in insisting on the importance of falsification of laws of nature
  (as he said, ``the scientific spirit requires a man to be at all times ready
  to dump his whole cartload of beliefs, the moment experience is against them''
  \cite[pp.~46--47]{Peirce:1955-4}).
  His version of Popper's evolutionary schema \eqref{eq:Popper} is
  \begin{equation*}
    \text{belief -- surprise -- doubt -- inquiry -- belief}
  \end{equation*}
  (as presented by Misak in \cite[p.~11]{Misak:2004}).
\end{remark}

\section{Universal prediction}
\label{sec:universal}

  \epigraph{There is the logical disjunction:
    Either an intrinsically improbable event will occur,
    or, the prediction will [\ldots] be verified.}%
  {\textit{Statistical Methods and Scientific Inference}\\ \textsc{Ronald Fisher}}\nocite{Fisher:1973}

\noindent
The following theorem says that there exists a universal prediction system,
in the sense that it is stronger than any other prediction system
if we ignore a multiplicative increase in the number of errors made.
\begin{theorem}\label{thm:basic}
  There is a \emph{universal} prediction system ${\mathcal U}$,
  in the sense that for any prediction system $\mathcal{L}$
  there exists a constant $c>0$ such that, for any $N$,
  \begin{equation}\label{eq:basic}
    \mathcal{L}_{\le N}\subseteq\mathcal{U}_{\le c N}.
  \end{equation}
\end{theorem}
\begin{proof}
  Let $\mathcal{L}^1,\mathcal{L}^2,\ldots$ be a recursive enumeration of all prediction systems;
  their component laws of nature will be denoted $(L^k_1,L^k_2,\ldots):=\mathcal{L}^k$.
  (Formally, we require the set
  $\{(\sigma,n,k)\in(\mathbf{X}\times\mathbb{2})^*\times\mathbb{N}^2\st\sigma\in L^k_n\}$
  to be recursively enumerable
  and the sequence $\mathcal{L}^1,\mathcal{L}^2,\ldots$ to contain all prediction systems.)
  For each $n\in\mathbb{N}$,
  define the $n$th component $U_n$ of $\mathcal{U}=(U_1,U_2,\ldots)$ as follows.
  Let the binary representation of $n$ be
  \begin{equation*}
    \left(
      a,0,1^{k-1}
    \right)
    =
    (a,0,1,\ldots,1),
  \end{equation*}
  where $a$ is a binary string (starting from 1) and the number of 1s in the $1,\ldots,1$ is $k-1\in\mathbb{N}_0$
  (this sentence is the definition of $a=a(n)$ and $k=k(n)$ in terms of $n$).
  If the binary representation of $n$ does not contain any 0s,
  $a$ and $k$ are undefined, and we set $U_n:=\emptyset$.
  Otherwise, set
  \begin{equation*}
    U_n
    :=
    L^k_{A},
  \end{equation*}
  where $A\in\mathbb{N}$ is the number whose binary representation is $a$.
  In other words, $\mathcal{U}$ consists of the components of $\mathcal{L}^k$, $k\in\mathbb{N}$;
  namely, $L^k_1$ is placed in $\mathcal{U}$ as $U_{3\times2^{k-1}-1}$ and then $L^k_2,L^k_3,\ldots$
  are placed at intervals of $2^k$:
  \begin{equation*}
    U_{3\times2^{k-1}-1+2^k(n-1)}
    =
    L^k_n,
    \quad
    n=1,2,\ldots.
  \end{equation*}
  It is easy to see that
  \begin{equation}\label{eq:strong}
    \mathcal{L}^k_{\le N}
    \subseteq
    \mathcal{U}_{\le 3\times2^{k-1}-1+2^k(N-1)},
  \end{equation}
  which is stronger than \eqref{eq:basic}.
\end{proof}

Let us fix a universal prediction system $\mathcal{U}$.
We can equivalently rewrite \eqref{eq:basic}
as the inclusion between the extreme terms of
\begin{equation}\label{eq:basic-bis}
  \Pi^{c N}_{\mathcal{U}}(s)
  =
  \left\{
    y\in\mathbb{2}
    \st
    (s,y)\notin\mathcal{U}_{\le c N}
  \right\}
  \subseteq
  \left\{
    y\in\mathbb{2}
    \st
    (s,y)\notin\mathcal{L}_{\le N}
  \right\}
  =
  \Pi^N_{\mathcal{L}}(s),
\end{equation}
for all situations $s$.
Intuitively, \eqref{eq:basic-bis} says that the prediction sets output by the universal prediction system
are at least as precise as the prediction sets output by any other prediction system $\mathcal{L}$
if we ignore a constant factor in specifying the level $N$.

In terms of the attained level \eqref{eq:pi},
Theorem~\ref{thm:basic} says that, as a function of $s$ and $y$,
$\pi^s_{\mathcal{U}}(y)$ does not exceed $\pi^s_{\mathcal{L}}(y)$
to within a constant factor.
Indeed, assuming that $c\in\mathbb{N}$ in \eqref{eq:basic},
\begin{align*}
  \pi^s_{\mathcal{L}}(y)
  &=
  \min
  \left\{
    n\in\mathbb{N}
    \st
    (s,y)\in L_n
  \right\}
  =
  \max
  \left\{
    N\in\mathbb{N}_0
    \st
    (s,y)\notin\mathcal{L}_{\le N}
  \right\}
  +
  1\\
  &\ge
  \max
  \left\{
    N\in\mathbb{N}_0
    \st
    (s,y)\notin\mathcal{U}_{\le c N}
  \right\}
  +
  1\\
  &=
  \frac{1}{c}
  \max
  \left\{
    N'\in c\mathbb{N}_0
    \st
    (s,y)\notin\mathcal{U}_{\le N'}
  \right\}
  +
  1\\
  &\ge
  \frac{1}{c}
  \max
  \left\{
    N'\in\mathbb{N}_0
    \st
    (s,y)\notin\mathcal{U}_{\le N'}
  \right\}\\
  &=
  \frac{1}{c}
  \min
  \left\{
    n\in\mathbb{N}
    \st
    (s,y)\in U_n
  \right\}
  -
  \frac{1}{c}
  =
  \frac{1}{c}
  \pi^s_{\mathcal{U}}(y)
  -
  \frac{1}{c},
\end{align*}
which implies
\begin{equation}\label{eq:pi-best}
  \pi^s_{\mathcal{L}}(y)
  \ge
  \frac{1}{2c}
  \pi^s_{\mathcal{U}}(y)
\end{equation}
when $\pi^s_{\mathcal{U}}(y)\ge2$;
and when $\pi^s_{\mathcal{U}}(y)=1$,
\eqref{eq:pi-best} follows from $\pi^s_{\mathcal{L}}(y)\ge1$.

If we are in a situation $s=(\sigma,x)$ and one of the two $\pi^s_{\mathcal{U}}(y)$
(corresponding to $y=0$ or $y=1$) is a small number,
we can predict the other label:
e.g., if $\pi^s_{\mathcal{U}}(0)$ is small,
we can predict that the label of $x$ is $1$,
and we then have Fisher's disjunction:
either our prediction is correct, or a rare event has occurred.

\section{Complexity of prediction systems and laws of nature}
\label{sec:complexity}

In this section we will see how the constant $c$ in Theorem~\ref{thm:basic}
depends on the prediction system $\mathcal{L}$.
The dependence will be in terms of the algorithmic complexity of $\mathcal{L}$,
which we will now define.

A \emph{description language for prediction systems} is a function $F$
mapping $\mathbb{2}^*$ to the set of all prediction systems such that the set
\[
  \left\{
    (d,\sigma,n)\in\mathbb{2}^*\times(\mathbf{X}\times\mathbb{2})^*\times\mathbb{N}
    \st
    \sigma \in F_n(d)
  \right\}
\]
is recursively enumerable,
where $F_n(d)$ is the $n$th law of nature in $F(d)$,
i.e., $F_n(d):=L_n$ when $F(d)=(L_1,L_2,\ldots)$.
Notice that the domain of $F$ is $\mathbb{2}^*$ rather than a subset of $\mathbb{2}^*$,
which is unusual for the theory of algorithmic complexity.
The \emph{effective domain} $\dom(F)$ of a description language $F$ for prediction systems is
\begin{equation}\label{eq:dom}
  \dom(F) := \{d\st F(d)\ne(\emptyset,\emptyset,\ldots)\}.
\end{equation}
A \emph{prefix-free description language for prediction systems}
is a description language $F$ for prediction systems
such that $\dom(F)$ is prefix-free.

The \emph{complexity} of a prediction system $\mathcal{L}$
with respect to a description language $F$ for prediction systems
is defined by
\[
  C_F(\mathcal{L})
  :=
  \min
  \left\{
    \left|d\right|
    \st
    \mathcal{L} = F(d)
  \right\},
\]
$\left|d\right|$ standing for the length of $d$.

\begin{theorem}\label{thm:complexity-of-prediction-systems}
  There is a description language $U$ for prediction systems that is universal
  in the sense that for any description language $F$ for prediction systems
  there exists a constant $c$ such that,
  for any prediction system $\mathcal{L}$,
  \begin{equation}\label{eq:complexity-of-prediction-systems}
    C_U(\mathcal{L})
    \le
    C_F(\mathcal{L}) + c.
  \end{equation}
  There is a prefix-free description language $U'$ for prediction systems that is universal
  in the sense that for any prefix-free description language $F$ for prediction systems
  there exists a constant $c$ such that,
  for any prediction system $\mathcal{L}$,
  \[
    C_{U'}(\mathcal{L})
    \le
    C_F(\mathcal{L}) + c.
  \]
\end{theorem}

\begin{proof}
  We will use the same (very standard) argument as in Theorem~\ref{thm:basic}
  and will only prove \eqref{eq:complexity-of-prediction-systems}.
  Let $F^k$, $k=1,2,\ldots$ be a recursive enumeration of the description languages for prediction systems
  (meaning that the set
  \[
    \left\{
      (d,\sigma,n,k)\in\mathbb{2}^*\times(\mathbf{X}\times\mathbb{2})^*\times\mathbb{N}^2
      \st
      \sigma \in F^k_n(d)
    \right\}
  \]
  is recursively enumerable
  and that each description language for prediction systems belongs to the sequence $F^1,F^2,\ldots$).
  Let $1^k0d$ serve as a description of $F^k(d)$ under $U$
  (where $1^k$ stands for the binary sequence $(1,\ldots,1)$ consisting of $k$ $1$s).
\end{proof}

Let us fix a universal description language $U$ for prediction systems,
call $C_U(\mathcal{L})$ the \emph{plain complexity} of $\mathcal{L}$,
and abbreviate $C_U(\mathcal{L})$ to $C(\mathcal{L})$.
Analogously, we fix a universal prefix-free description language $U'$,
call $C_{U'}(\mathcal{L})$ the \emph{prefix complexity} of $\mathcal{L}$,
and abbreviate $C_{U'}(\mathcal{L})$ to $K(\mathcal{L})$.

The following theorem makes \eqref{eq:basic} uniform in $\mathcal{L}$
showing how $c$ depends on $\mathcal{L}$.
\begin{theorem}\label{thm:precise}
  There is a constant $c>0$ such that, for any prediction system $\mathcal{L}$ and any $N\in\mathbb{N}$,
  the universal prediction system $\mathcal{U}$ satisfies
  \begin{equation}\label{eq:precise}
    \mathcal{L}_{\le N}
    \subseteq
    \mathcal{U}_{\le c2^{K(\mathcal{L})}N}.
  \end{equation}
\end{theorem}
\begin{proof}
  Define a prediction system $\mathcal{V}$ as the sequence $(V_1,V_2,\ldots)$ of laws of nature
  such that $V_n:=U'_{n'}(d)$,
  where $U'$ is the universal prefix-free description language for prediction systems,
  and $n'\in\mathbb{N}$ and $d\in\mathbb{2}^*$ are defined given $n$ as follows:
  \begin{itemize}
  \item
    $d$ is the suffix (if it exists) of the binary representation of $n$
    such that $\overleftarrow{d}$ belongs to $\dom(U')$
    (where $\overleftarrow{d}$ is the mirror image of $d$:
    $\left|\overleftarrow{d}\right|=\left|d\right|$
    and the bits of $\overleftarrow{d}$ are the same as the bits of $d$
    but written in the opposite order);
  \item
    the binary representation of $n'$ is the prefix (if non-empty) of the binary representation of $n$
    left after removing its suffix $d$.
  \end{itemize}
  It is clear that such $n'$ and $d$ are unique when they exist;
  and when they do not exist, set $V_n:=\emptyset$.
  Then the modification
  \begin{equation*}
    U'_n(d)
    \subseteq
    \mathcal{V}_{\le n 2^{\left|d\right|} + 2^{\left|d\right|} - 1}
  \end{equation*}
  of \eqref{eq:strong} implies, for any prediction system $\mathcal{L}$,
  \begin{equation*}
    \mathcal{L}_n
    \subseteq
    \mathcal{V}_{\le n 2^{K(\mathcal{L})} + 2^{K(\mathcal{L})} - 1}
  \end{equation*}
  (take as $d$ the shortest description of $\mathcal{L}$ under $U'$).
  This implies that \eqref{eq:precise} holds for some prediction system $\mathcal{V}$ in place of $\mathcal{U}$,
  which, when combined with the statement of Theorem~\ref{thm:basic},
  implies that \eqref{eq:precise} holds for our chosen universal prediction system $\mathcal{U}$.
\end{proof}

Specializing the notions of plain and prefix complexity for a prediction system
to prediction systems of type $\mathcal{L}=(L,L,\ldots)$,
we obtain the notions of plain and prefix complexity for a law of nature:
\begin{align*}
  C(L)
  &:=
  C((L,L,\ldots)),\\
  K(L)
  &:=
  K((L,L,\ldots)).
\end{align*}
However, since the notion of algorithmic complexity of a law of nature
will be used in the next section as a basis for defining the complexity of time,
we will also spell out the simpler direct definition.

A \emph{description language for laws of nature} is a function $F$ mapping $\mathbb{2}^*$
to the set of prefix-free subsets of $(\mathbf{X}\times\mathbb{2})^*$ such that the set
\[
  \left\{
    (d,\sigma)\in\mathbb{2}^*\times(\mathbf{X}\times\mathbb{2})^* \st \sigma \in F(d)
  \right\}
\]
is recursively enumerable.
We will usually omit ``for laws of nature''.
Notice that, for any description language $F$ and any \emph{description} $d\in\mathbb{2}^*$,
$F(d)$ is a law of nature
(formally, we use ``description'' to mean elements of $\mathbb{2}^*$;
informally, descriptions serve as arguments for description languages).
The \emph{effective domain} $\dom(F)$ of a description language $F$
is
\[
  \dom(F) := \{d\st F(d)\ne\emptyset\}.
\]
A \emph{prefix-free description language}
is a description language $F$
such that $\dom(F)$ is prefix-free.

The \emph{complexity} of a law of nature $L$ with respect to a description language $F$
is defined by
\[
  C_F(L)
  :=
  \min
  \left\{
    \left|d\right|
    \st
    L = F(d)
  \right\}.
\]
The analogue of Theorem~\ref{thm:complexity-of-prediction-systems}
continues to hold for laws of nature;
we fix a universal description language $U$, call $C_U(L)$
the \emph{plain complexity} of $L$,
and abbreviate $C_U(L)$ to $C(L)$;
and we fix a universal prefix-free description language $U'$, call $C_{U'}(L)$
the \emph{prefix complexity} of $L$,
and abbreviate $C_{U'}(L)$ to $K(L)$.

This is a corollary of Theorem~\ref{thm:precise} for laws of nature:
\begin{corollary}\label{cor:precise}
  There is a constant $c>0$ such that, for any law of nature $L$,
  the universal prediction system $\mathcal{U}$ satisfies
  \begin{equation}\label{eq:for-laws}
    L\subseteq \mathcal{U}_{\le c2^{K(L)}}.
  \end{equation}
\end{corollary}
\begin{proof}
  We again regard laws of nature $L$ as a special case of prediction systems
  identifying $L$ with $\mathcal{L}:=(L,L,\ldots)$.
  It remains to apply Theorem~\ref{thm:precise} to $\mathcal{L}$ setting $N:=1$.
\end{proof}

A simple counting argument shows that the dependence of the right-hand side of~\eqref{eq:precise}
on the complexity of $\mathcal{L}$ is approximately correct and cannot be significantly improved
(if the difference between plain and prefix complexities is ignored).
To state this argument in its strongest form,
we will introduce a new piece of notation:
for each infinite data sequence $\omega\in(\mathbf{X}\times\mathbb{2})^{\infty}$,
\[
  \Sigma(\omega)
  :=
  \left\{
    \omega^l
    \st
    l\in\mathbb{N}_0
  \right\}
\]
is the set of all finite prefixes of $\omega$.
Theorem~\ref{thm:precise} says that there is a constant $c>0$ such that,
for any $K,N\in\mathbb{N}$,
any infinite data sequence $\omega$,
and any prediction system $\mathcal{L}$ satisfying $K(\mathcal{L})\le K$,
\begin{equation}\label{eq:precise-2}
  \mathcal{L}_{\le N} \cap \Sigma(\omega)
  \subseteq
  \mathcal{U}_{\le c2^{K}N} \cap \Sigma(\omega).
\end{equation}
The inclusion in \eqref{eq:precise-2}
compares the predictive powers of $\mathcal{L}$ and $\mathcal{U}$
only along the infinite data sequence $\omega$.

\begin{theorem}\label{thm:precise-opposite}
  There is a constant $c>0$ such that,
  for any $K,N\in\mathbb{N}$ and any infinite data sequence $\omega$,
  there exists a prediction system $\mathcal{L}$ satisfying $C(\mathcal{L})\le K$ and
  \begin{equation}\label{eq:tight}
    \mathcal{L}_{\le N} \cap \Sigma(\omega)
    \nsubseteq
    \mathcal{U}_{\le c2^{K}N} \cap \Sigma(\omega).
  \end{equation}
\end{theorem}
\begin{proof}
  Let $\mathcal{L}^k$, $k\in\mathbb{N}$,
  be the strong prediction system such that $\mathcal{L}^k_{<\infty}$
  (defined by \eqref{eq:limit})
  consists of finite data sequences
  whose length is divisible by $2^{k-1}$ but not divisible by $2^k$
  (what is essential is that different $\mathcal{L}^k$
  make errors on disjoint sets of finite data sequences).
  Take any $K,N\in\mathbb{N}$ and any $\omega\in(\mathbf{X}\times\mathbb{2})^{\infty}$.
  Set $K':=K-a$ for some constant $a\in\mathbb{N}$, to be chosen later.
  The set $\mathcal{U}_{\le 2^{K'}N}\cap\Sigma(\omega)$
  contains at most $2^{K'}N$ elements;
  therefore, \eqref{eq:tight} will be satisfied for $c:=2^{-a}$,
  for some $\mathcal{L}:=\mathcal{L}^k$
  and $k\le2^{K'}+1$.
  It remains to notice that $C(\mathcal{L}^k)\le K'+O(1)\le K$
  provided $a$ is sufficiently large.
\end{proof}

We have the following corollary of Theorem~\ref{thm:precise-opposite} for laws of nature
showing the tightness (to within the difference between $C$ and $K$) of Corollary~\ref{cor:precise}.
\begin{corollary}
  There is a constant $c>0$ such that,
  for any $K\in\mathbb{N}$ and any infinite data sequence $\omega$,
  there exists a law of nature $L$ satisfying $C(L)\le K$ and
  \begin{equation*}
    L \cap \Sigma(\omega)
    \nsubseteq
    \mathcal{U}_{\le c2^K} \cap \Sigma(\omega).
  \end{equation*}
\end{corollary}

\begin{proof}
  Specialize Theorem~\ref{thm:precise-opposite} to the case $N:=1$
  and define $L$ to be the first element of the prediction system $\mathcal{L}$.
  The additive constant implicit in the definition of the plain complexity $C(L)$
  can be incorporated into the constant $c$,
  as we did in the proof of Theorem~\ref{thm:precise-opposite}.
\end{proof}

Analogously to~\eqref{eq:basic} and~\eqref{eq:basic-bis},
we can rewrite \eqref{eq:precise} and \eqref{eq:for-laws} as
\begin{align}
  \Pi^{c 2^{K(\mathcal{L})} N}_{\mathcal{U}}(s)
  \subseteq
  \Pi^N_{\mathcal{L}}(s)
  \label{eq:precise-bis}\\
\intertext{and}
  \Pi^{c 2^{K(L)}}_{\mathcal{U}}(s)
  \subseteq
  \Pi_{L}(s),
  \label{eq:for-laws-bis}
\end{align}
respectively, for all situations $s$;
\eqref{eq:precise-bis} and \eqref{eq:for-laws-bis} indicate the dependence of the constant factor
in~\eqref{eq:basic-bis} on~$\mathcal{L}$.

\begin{remark}[\cite{Shen:PC}]\label{rem:prefix-1}
  This is a natural modification of our definition of prefix-free description languages:
  a description language $F$ for laws of nature is \emph{prefix-correct} if, for all $d_1,d_2\in\mathbb{2}^*$,
  \[
    d_1 \sqsubseteq d_2
    \Longrightarrow
    F(d_1) \subseteq F(d_2).
  \]
  There is a universal prefix-correct description language $U''$ in the sense that $C_{U''}\le C_F+O(1)$
  for any prefix-correct description language $F$.
  Let us fix such a $U''$ and call $K'(F):=C_{U''}(F)$ the \emph{intermediate complexity} of $F$.
\end{remark}

\section{Time complexity of finite data sequences}
\label{sec:time-complexity}

The \emph{plain time complexity} and \emph{prefix time complexity} of a finite data sequence $\sigma$ are defined by
\begin{align}
  \CCC(\sigma)
  &:=
  \min_{L\ni\sigma}
  C(L),
  \label{eq:C}\\
  \KKK(\sigma)
  &:=
  \min_{L\ni\sigma}
  K(L),
  \label{eq:K}
\end{align}
respectively,
where $L$ ranges over the laws of nature.
(We will explain the terminology later in this section.)
We have to modify the notation $C$ and $K$ slightly
since we would like to be able to use the standard notation $C(\sigma)$ and $K(\sigma)$
for the Kolmogorov complexity (plain and prefix) of $\sigma$;
we will also use $C(n)$ and $K(n)$ to denote the Kolmogorov complexity
(plain or prefix) of an integer $n$.

The following simple result is useful for discussing the interpretation of $\CCC(\sigma)$ and $\KKK(\sigma)$.

\begin{theorem}\label{thm:trivial}
  For any finite data sequence $\sigma$,
  \begin{align}
    \CCC(\sigma) &\le C(\left|\sigma\right|) + O(1),\label{eq:C-trivial}\\
    \KKK(\sigma) &\le K(\left|\sigma\right|) + O(1).\label{eq:K-trivial}
  \end{align}
\end{theorem}

\begin{proof}
  If $F$ is any description language (prefix-free or not) for nonnegative integers,
  we can define a description language $F'$ for laws of nature
  by setting
  \[
    F'(d)
    :=
    \begin{cases}
      (\mathbf{X}\times\mathbb{2})^{F(d)} & \text{if $F(d)$ is defined}\\
      \emptyset & \text{if not}.
    \end{cases}
  \]
  Since
  \begin{align*}
    C_{F'}\left((\mathbf{X}\times\mathbb{2})^{\left|\sigma\right|}\right) &= C_F(\left|\sigma\right|),
  \end{align*}
  \eqref{eq:C-trivial} follows by setting $F:=U$
  and \eqref{eq:K-trivial} follows by setting $F:=U'$.
\end{proof}

Theorem~\ref{thm:trivial} gives a trivial bound on the time complexity of $\sigma$:
it is the complexity of the length of $\sigma$
(i.e., of the time of the last observation in $\sigma$
assuming that the observations are taken at times $1,2,\ldots$).
We can say that both $\CCC(\sigma)$ and $\KKK(\sigma)$ measure
the complexity of the time of the last observation in $\sigma$
when we are given the observations themselves as an oracle
(with the observations disclosed sequentially, so that we can't just count them).
As we will see later (see Theorems~\ref{thm:connection-1} and~\ref{thm:K-M} below),
these measures of complexity can be used to determine
whether being in the situation of having just observed
the last observation in $\sigma$
is a rare event\footnote{%
  For the reader familiar with Shafer's (\cite{Shafer:1996art}, Section~1.7) distinction
  between Humean and Moivrean events,
  we are talking about events of the former kind.}.
For the purpose of prediction,
having such a measure of complexity is important
since our prediction system can be forgiven for giving a wrong prediction
when a rare event happens
(cf.\ the epigraph about ``Fisher's disjunction'' to Section~\ref{sec:universal}).
\Extra{This explanation might look cryptic now, but we will return to it later.}

\begin{remark}
  The length of a finite data sequence $\sigma$
  can be interpreted as the physical time of the last observation in $\sigma$.
  In probability theory,
  physical time is often changed;
  e.g., it can be replaced by intrinsic time reflecting the intensity
  at which various events happen
  (in a probability-free setting,
  this was done in, e.g., \cite{\CTIV},
  where physical time was replaced by quadratic variation).
  The stopping times (see Remark~\ref{rem:stopping-time}) corresponding to physical time
  consist of all finite data sequences of the same length.
  For the more general notion of time,
  we can regard the last observations in the finite data sequences
  in an arbitrary stopping time (law of nature)
  as happening at the same moment in time.
  This is another justification for calling \eqref{eq:C}--\eqref{eq:K}
  the time complexity of $\sigma$.
\end{remark}

\begin{remark}\label{rem:Shen-2}
  In the usual jargon of Kolmogorov complexity,
  we can say that the complexity (either plain or prefix) of $\sigma$
  is the minimal complexity (of the same kind) of a binary program
  that enumerates some prefix-free set containing $\sigma$.
\end{remark}

The following theorem describes a connection with the universal prediction system;
remember that $\log$ is binary logarithm.
\begin{theorem}\label{thm:connection-1}
  There is a constant $c>0$ such that, for all $N$,
  \begin{equation}\label{eq:connection-1}
    \{\sigma\st\CCC(\sigma)\le\log N-c\}
    \subseteq
    \mathcal{U}_{\le N}
    \subseteq
    \{\sigma\st\CCC(\sigma)\le \log N+c\}.
  \end{equation}
\end{theorem}

\begin{proof}
  To check the left-hand inclusion in~\eqref{eq:connection-1},
  it suffices to define a prediction system $\mathcal{L}$
  such that, for all finite data sequences $\sigma$,
  $\sigma\in\mathcal{L}_{\le 2^{k+1}}$ where $k:=\CCC(\sigma)$.
  Let $U$ be the universal description language for laws of nature: $C_U=\CCC$.
  We can set $\mathcal{L}:=(L_1,L_2,\ldots)$,
  where $L_n$ is defined to be $F(d)$ for $d$ obtained from the binary representation of $n$
  by removing the leading $1$.

  To check the right-hand inclusion in~\eqref{eq:connection-1},
  it suffices to define a description language $F$ for laws of nature
  such that $C_F(\sigma)\le\log N$ whenever $\sigma\in\mathcal{U}_{\le N}$, for any $N$.
  Define $F(d)$, where $d\in\mathbb{2}^*$,
  as $U_n$, where $n$ is the natural number whose binary representation is 1 followed by $d$.
  If $\sigma\in\mathcal{U}_{\le N}$,
  $\sigma$ will belong to a law of nature whose description is of length at most $\lfloor\log N\rfloor$,
  which completes the proof of this inclusion.
\end{proof}

We can interpret \eqref{eq:connection-1}
by saying that $\mathcal{U}_{\le N}$ coincides with $\{\sigma\st\CCC(\sigma)\le\log N\}$
if we are allowed to vary the threshold $\log N$ by adding a constant (positive or negative);
this qualification is natural as time complexity is defined only to within an additive constant.

The next result gives an even simpler connection.
\begin{theorem}
  When $s\in(\mathbf{X}\times\mathbb{2})^*\times\mathbf{X}$ ranges over the situations
  and $y\in\mathbb{2}$ over the labels,
  \begin{equation*}
    \log\pi^s_{\mathcal{U}}(y)
    =
    \CCC((s,y)) + O(1).
  \end{equation*}
\end{theorem}

\begin{proof}
  This follows immediately from \eqref{eq:connection-1}:
  \begin{align*}
    \log\pi^s_{\mathcal{U}}(y)
    &=
    \min\{\log n \st (s,y) \in U_n\}
    =
    \min\{\log N \st (s,y) \in \mathcal{U}_{\le N}\}\\
    &=
    \min\{\log N \st \CCC((s,y)) \le \log N\} + O(1)
    =
    \CCC((s,y)) + O(1),
  \end{align*}
  where $n$ and $N$ range over $\mathbb{N}$.
\end{proof}

And the following theorem gives obvious connections between the two complexities.

\begin{theorem}
  \begin{align*}
    \CCC(\sigma) &\le \KKK(\sigma) + O(1)\\
    \KKK(\sigma) &\le \CCC(\sigma) + 2\log \CCC(\sigma) + O(1).
  \end{align*}
\end{theorem}

\begin{proof}
  The first inequality follows from the fact that a prefix-free description language
  is a description language.
  The second inequality follows from the fact that any description $d$
  can be turned into a prefix-free description by prefixing it
  by the following prefix-free description of the length $\left|d\right|$ of $d$:
  double each bit of the binary representation of $\left|d\right|$
  and add the string $(0,1)$ as suffix.
\end{proof}

\begin{remark}[\cite{Shen:PC}]\label{rem:prefix-2}
  We can complement \eqref{eq:C} and \eqref{eq:K}
  by $\KKK'(\sigma):=\min_{L\ni\sigma}K'(L)$,
  where $K'$ is as defined in Remark~\ref{rem:prefix-1}.
  We will refer to $\KKK'(\sigma)$ as the \emph{intermediate time complexity} of $\sigma$;
  notice that
  \[
    \CCC-O(1) \le \KKK' \le \KKK+O(1).
  \]
\end{remark}

\section{\emph{A priori} time semimeasure}
\label{sec:semimeasure}

We can also define an analogue of Levin's \emph{a priori} semimeasure
(see, e.g., \cite{Shen:2015}, Section~7.33)
for time.
A \emph{time semimeasure} is a function $P:(\mathbf{X}\times\mathbb{2})^*\to[0,1]$ such that,
for all infinite data sequences $\omega$,
\[
  \sum_{l=0}^{\infty}
  P(\omega^l)
  \le
  1.
\]

\begin{theorem}
  There is a largest to within a constant factor lower semicomputable time semimeasure.
\end{theorem}

\begin{proof}
  It is easy to check that there exists a sequence $P_k$, $k=1,2,\ldots$,
  of semicomputable time semimeasures that is \emph{jointly lower semicomputable},
  in the sense of the function $(k,\sigma)\mapsto P_k(\sigma)$ being lower semicomputable,
  and \emph{universal},
  in the sense of containing every lower semicomputable time semimeasure.
  For any such sequence, the average
  \[
    \MMM
    :=
    \sum_{k=1}^{\infty}
    2^{-k} P_k
  \]
  will be a largest to within a constant factor lower semicomputable time semimeasure.
\end{proof}

Let us fix a largest to within a constant factor lower semicomputable time semimeasure $\MMM$
and call it the \emph{a priori} time semimeasure.
We will use the notation $M$ for the standard \emph{a priori} semimeasure on $\mathbb{N}_0$;
it is well known that $-\log M$ coincides with prefix complexity $K$ to within an additive constant
(see, e.g., \cite{Shen:2015}, Theorem~7.29).
For the time counterparts of $M$ and $K$ we will only state a weaker result.

\begin{theorem}\label{thm:K-M}
  $\CCC-O(1) \le -\log\MMM \le \KKK+O(1)$.
\end{theorem}

\begin{proof}
  To check the inequality $-\log\MMM \le \KKK + O(1)$,
  it suffices to check that $2^{-\KKK}$ is a time semimeasure
  (its lower semicomputability follows from the upper semicomputability of $\KKK$).
  Fix an infinite data sequence $\omega$.
  For each $l$, let $L_l$ be the simplest, in the sense of $\KKK$,
  law of nature containing $\omega^l$.
  By the definition of a law of nature all $L_l$ are pairwise distinct,
  and so we have
  \begin{equation}\label{eq:K-sum}
    \sum_{l=0}^{\infty}
    2^{-\KKK(\omega^l)}
    =
    \sum_{l=0}^{\infty}
    2^{-K(L_l)}
    \le
    \sum_L
    2^{-K(L)}
    \le
    1,
  \end{equation}
  where the last sum is over all laws of nature $L$
  (the last inequality is obvious,
  but a detailed proof can be found in, e.g., \cite{Shen:2015}, Theorem~7.27).

  To check the opposite inequality $\CCC \le -\log\MMM + O(1)$,
  it suffices to define a description language $F$ for laws of nature
  such that $\min_{L\ni\sigma}C_F(L) \le -\log\MMM(\sigma) + O(1)$.
  For each threshold $k\in\mathbb{N}_0$, we can enumerate (in a computable manner)
  all data sequences $\sigma$ satisfying $\MMM(\sigma)>2^{-k}$
  (as $\MMM$ is lower semicomputable,
  we will be able to detect $\MMM(\sigma)>2^{-k}$ eventually);
  let $\sigma_1,\sigma_2,\ldots$ be such an enumeration
  (the sequence $\sigma_1,\sigma_2,\ldots$ can be finite and even empty, as it is for $k=0$).
  Order the $2^k$ binary strings in $\mathbb{2}^k$ lexicographically.
  For $n=1,2,\ldots$:
  assign to $\sigma_n$ as its description the smallest element of $\mathbb{2}^k$
  that does not serve as description for any of $\sigma_1,\ldots,\sigma_{n-1}$
  that is comparable with $\sigma_n$ w.r.\ to $\sqsubseteq$
  (in particular, $\sigma_1$ has $0^k=(0,\ldots,0)$ as its description).
  Since, for each infinite data sequence $\omega$,
  $\MMM(\omega^l)>2^{-k}$ holds for at most $2^k$ (and even $2^k-1$) $l$s,
  we will never run out of descriptions when following this procedure.
  Define $F(d)$, where $d\in\mathbb{2}^k$,
  to be the set of all $\sigma$ having $d$ as their description;
  by construction, $F(d)$ is a law of nature
  and $F$ is a description language
  (remember that the procedure is repeated for all $k\in\mathbb{N}_0$).
  Since
  \[
    \MMM(\sigma)>2^{-k}
    \Longrightarrow
    C_F(\sigma)\le k
  \]
  for all $\sigma\in(\mathbf{X}\times\mathbb{2})^*$ and $k\in\mathbb{N}_0$,
  we have $C_F\le-\log\MMM+1$ and, therefore, $\CCC\le-\log\MMM+O(1)$.
\end{proof}

In fact, Alexander Shen pointed out that the standard connection between $M$ and $K$,
$K=-\log M+O(1)$, does not carry over to their time counterparts.
(Shen's observation is a version of another standard result in the theory of Kolmogorov complexity.)

\begin{theorem}[A.~Shen]\label{thm:Shen}
  It is not true that $\KKK = -\log\MMM+O(1)$.
\end{theorem}

\begin{proof}
  Suppose that, in fact, $\KKK = -\log\MMM+O(1)$.
  Fix an object $\mathbf{x}\in\mathbf{X}$ and two labels $a,b\in\mathbf{Y}$.
  Set $A:=(\mathbf{x},a)\in\mathbf{Z}$,
  $\alpha:=(A,A,\ldots)\in\mathbf{Z}^{\infty}$, and $B:=(\mathbf{x},b)\in\mathbf{Z}$.
  For each $n\in\mathbb{N}$,
  consider the following $n$ finite data sequences:
  \[
    \alpha^n B \alpha^1, \quad
    \alpha^n B \alpha^2, \ldots, \quad
    \alpha^n B \alpha^n.
  \]
  Since there is a time semimeasure $P$ satisfying $P(\alpha^n B \alpha^k)=1/n$,
  for all $n\in\mathbb{N}$ and all $k=1,\ldots,n$,
  we have $\MMM(\alpha^n B \alpha^k)\ge1/(c n)$,
  for all $n\in\mathbb{N}$ and all $k=1,\ldots,n$,
  $c$ standing for a positive universal constant
  (with different occurrences of $c$ referring to possibly different positive universal constants).
  By our assumption,
  $2^{-\KKK(\alpha^n B \alpha^k)}\ge1/(c n)$,
  for all $n\in\mathbb{N}$ and all $k=1,\ldots,n$.
  Remember that $\sum_L 2^{-K(L)}\le1$, where the sum is over all laws of nature
  (we have already used this: see \eqref{eq:K-sum}).
  The series $\sum_L 2^{-K(L)}$ contains at least $n$ terms $2^{-K(L)}\ge1/(c n)$
  (since laws of natures containing $\alpha^n B \alpha^k$ and $\alpha^n B \alpha^{k'}$
  are necessarily different when $k\ne k'$).
  The series is positive, and so its sum will not change if we rearrange its terms.
  Let us sort them in the decreasing order.
  The $n$th largest term will be at least $1/(c n)$,
  and therefore $\sum_n 1/n=\infty$ implies $\sum_L 2^{-K(L)}=\infty$.
  This contradiction concludes the proof.
\end{proof}

\begin{remark}[\cite{Shen:2017}]\label{rem:Andreev}
  As shown by Mikhail Andreev,
  it is also not true that $\KKK' = -\log\MMM+O(1)$,
  where $\KKK'$ is intermediate time complexity, as defined in Remark~\ref{rem:prefix-2}.
  The proof is much more difficult and can be found in \cite{Shen:2017}.
\end{remark}

\section{Time randomness}
\label{sec:randomness}

In the usual theory of Kolmogorov complexity
the notion of algorithmic randomness is as important as that of algorithmic complexity
(and perhaps was the main motivation behind Kolmogorov's introduction of algorithmic complexity).
There are many versions of algorithmic randomness,
and in this paper we will briefly discuss only the time analogue
of Kolmogorov's original definition $\left|\sigma\right|-C(\sigma)$ of the randomness deficiency
of a binary string $\sigma$ of length $\left|\sigma\right|$
(given, somewhat implicitly, in \cite{Kolmogorov:1965}, Section~4)
and, later on (see Theorem~\ref{thm:universal-randomness-type}),
the time analogue of Martin-L\"of's \cite{Martin-Lof:1966} definition of randomness.

The \emph{time randomness deficiency} of a finite data sequence $\sigma\in(\mathbf{X}\times\mathbb{2})^*$ is defined to be
\[
  \DDD(\sigma):=\log\left|\sigma\right|-\CCC(\sigma).
\]
(We take $\log\left|\sigma\right|$ instead of Kolmogorov's $\left|\sigma\right|$ in view of Theorem~\ref{thm:trivial}:
whereas the trivial upper bound on plain Kolmogorov complexity is $C(\sigma)\le\left|\sigma\right|+O(1)$,
the trivial upper bound on plain time complexity is
\[
  \CCC(\sigma)
  \le
  C(\left|\sigma\right|) + O(1)
  \le
  \log\left|\sigma\right| + O(1).)
\]

Informally, we can rewrite \eqref{eq:connection-1} as
\begin{equation*}
  \mathcal{U}_{\le N}
  \approx
  \{\sigma\st\CCC(\sigma)\le\log N\}.
\end{equation*}
We could have defined the universal prediction system by
\begin{equation*}
  \mathcal{U}'_m
  :=
  \{\sigma\st\CCC(\sigma)\le m\}
\end{equation*}
(with $m$ in place of $\log N$).
This definition would be especially useful in situations without noise
where we can expect to make a finite number of prediction errors over an infinite data sequence.
In situations where there is noise at a more or less constant level for each observation
(which is typical under the assumption, prevalent in machine learning and nonparametric statistics,
that the observations are independent and identically distributed),
it may be more useful to replace $\CCC$ by $\DDD$ and set, for each threshold $m\in\mathbb{N}_0$,
\begin{equation*}
  \Delta_m
  :=
  \{\sigma\st\DDD(\sigma)>m\}.
\end{equation*}
The corresponding prediction sets are
\begin{equation*}
  \Pi_{\Delta_m}(s)
  :=
  \left\{
    y\in\mathbb{2}
    \st
    (s,y)\notin\Delta_m
  \right\}
  =
  \left\{
    y\in\mathbb{2}
    \st
    \DDD((s,y))\le m
  \right\}.
\end{equation*}
In a situation $s=(\sigma,x)$,
the prediction system $\Delta_m$ predicts that the label $y\in\mathbb{2}$ of $x$
will be an element of $\Pi_{\Delta_m}(s)$.
The following simple result shows that the rate at which this prediction system makes errors is less than $2^{-m}$.

\begin{theorem}\label{thm:time-randomness}
  For each infinite data sequence $\omega=((x_1,y_1),(x_2,y_2),\ldots)$, each $l\in\mathbb{N}$, and each $m\in\mathbb{N}_0$,
  \[
    \left|
      \left\{
        i\in\{1,\ldots,l\}
        \st
        y_i\notin\Pi_{\Delta_m}(\omega^{i-1},x_i)
      \right\}
    \right|
    =
    \left|
      \left\{
        i\in\{1,\ldots,l\}
        \st
        \omega^i\in\Delta_m
      \right\}
    \right|
    <
    2^{-m} l.
  \]
\end{theorem}

\begin{proof}
  If the prediction system $\Delta_m$ makes an error when predicting $y_i$,
  i.e., $y_i\notin\Pi_{\Delta_m}(\omega^{i-1},x_i)$, we have $\DDD(\omega^i)>m$, and so
  \[
    \CCC(\omega^i)
    <
    \log i - m
    \le
    \log l - m.
  \]
  The number of such $i$ does not exceed the number of all descriptions of length less than $\log l - m$,
  i.e., does not exceed $2^{\log l-m}-1<2^{-m}l$.
\end{proof}

In the rest of this section we will explore more systematically prediction systems
of the type $\Delta_m$.
(Notice that, formally, they are not even weak prediction systems as defined in Section~\ref{sec:weak}.)
A \emph{randomness-type prediction system} is a jointly enumerable family $\Lambda$
of sets $\Lambda_m\subseteq(\mathbf{X}\times\mathbb{2})^*$ of finite data sequences such that:
\begin{itemize}
\item
  $\Lambda_m$ are nested: $\Lambda_0\supseteq\Lambda_1\supseteq\Lambda_2\supseteq\cdots$;
\item
  for all $m\in\mathbb{N}_0$, $l\in\mathbb{N}$, and $\omega\in(\mathbf{X}\times\mathbb{2})^{\infty}$,
  \begin{equation}\label{eq:randomness-condition}
    \left|
      \left\{
        i\in\{1,\ldots,l\}
        \st
        \omega^i\in\Lambda_m
      \right\}
    \right|
    \le
    2^{-m} l.
  \end{equation}
\end{itemize}
Theorem~\ref{thm:time-randomness} says that $\Delta$ is a randomness-type prediction system.
It is easy to see that there is a universal randomness-type prediction system:

\begin{theorem}\label{thm:universal-randomness-type}
  There exists a randomness-type prediction system $\mathcal{D}$ such that,
  for any randomness-type prediction system $\Lambda$,
  there exists $c\in\mathbb{N}$ such that, for all $m\in\mathbb{N}_0$,
  $\Lambda_{m+c}\subseteq\mathcal{D}_{m}$.
\end{theorem}

\begin{proof}
  Notice that we can enumerate all randomness-type prediction systems $\Lambda^1,\Lambda^2,\ldots$,
  in the sense that there is a recursively enumerable set
  \begin{equation*}
    \Lambda \subseteq (\mathbf{X}\times\mathbb{2})^* \times \mathbb{N}^2
  \end{equation*}
  such that:
  \begin{enumerate}
  \item\label{it:1a}
    For any $k$, the sequence $(\Lambda^k_m)_{m=1}^{\infty}$, where
    \[
      \Lambda_m^k
      :=
      \left\{
        \sigma \in (\mathbf{X}\times\mathbb{2})^*
        \st
        (\sigma,m,k) \in \Lambda
      \right\}
    \]
    is a randomness-type prediction system.
  \item
    Any randomness-type prediction system coincides,
    for some $k$, with the sequence $(\Lambda^k_m)_{m=1}^{\infty}$.
  \end{enumerate}
  (The existence of such $\Lambda$ follows
  from the existence of such a set $\Lambda'$ when item~\ref{it:1a} is ignored
  and the fact that we can enumerate the elements of $\Lambda'$ one by one
  including each of them into $\Lambda$ if and only if the inclusion does not violate item~\ref{it:1a}.)
  We can then combine all these randomness-type prediction systems into $\mathcal{D}$
  setting
  \begin{equation}\label{eq:setting}
    \mathcal{D}_m
    :=
    \bigcup_{k=1}^{\infty}
    \Lambda^k_{m+k}.
  \end{equation}
  We will get a randomness-type prediction system,
  since
  \begin{align*}
    \left|
      \left\{
        i\in\{1,\ldots,l\}
        \st
        \omega^i\in\mathcal{D}_m
      \right\}
    \right|
    &=
    \left|
      \bigcup_{k=1}^{\infty}
      \left\{
        i\in\{1,\ldots,l\}
        \st
        \omega^i\in\Lambda^k_{m+k}
      \right\}
    \right|\\
    &\le
    \sum_{k=1}^{\infty}
    \left|
      \left\{
        i\in\{1,\ldots,l\}
        \st
        \omega^i\in\Lambda^k_{m+k}
      \right\}
    \right|\\
    &\le
    \sum_{k=1}^{\infty}
    2^{-m-k} l
    =
    2^{-m} l,
  \end{align*}
  and this system is obviously universal.
\end{proof}

Let us fix a randomness-type prediction system $\mathcal{D}$
satisfying the condition in Theorem~\ref{thm:universal-randomness-type}
and call it the \emph{universal randomness-type prediction system};
set, for any situation $s$ and any $m\in\mathbb{N}_0$,
\[
  \Pi_{\mathcal{D}_m}(s)
  :=
  \left\{
    y \in \mathbb{2}
    \st
    (s,y) \notin \mathcal{D}_m
  \right\}.
\]
A crude connection of $\mathcal{D}$ with our previous definition
is given in the following theorem.

\begin{theorem}\label{thm:randomness-complexity}
  There exists $c>0$ such that,
  for any finite data sequence $\sigma\in(\mathbf{X}\times\mathbb{2})^{l-1}$ (for any $l\in\mathbb{N}$),
  any $x\in\mathbf{X}$, and any threshold $m\in\mathbb{N}$,
  \begin{equation}\label{eq:randomness-complexity}
    \Pi^{c l 2^{-m} m^2}_{\mathcal{U}}((\sigma,x))
    \subseteq
    \Pi_{\mathcal{D}_m}((\sigma,x)).
  \end{equation}
\end{theorem}

This theorem asserts that the prediction set output by the universal prediction system
is at least as precise as the prediction set output by the universal randomness-type prediction system
if we increase slightly the allowed percentage of errors:
from $2^{-m}$ to $c 2^{-m} m^2$.
It involves not just multiplying by a constant
(as in, e.g., \eqref{eq:basic-bis})
but also the term $m^2$, which is logarithmic in the allowed percentage of errors $2^{-m}$ for $\mathcal{D}_m$.

By Theorem~\ref{thm:time-randomness},
Theorem~\ref{thm:randomness-complexity} will stay true
if we replace the right-hand side $\Pi_{\mathcal{D}_m}((\sigma,x))$
of \eqref{eq:randomness-complexity} by $\Pi_{\Delta_m}((\sigma,x))$;
moreover,
\[
  \Pi_{\mathcal{D}_{m}}((\sigma,x))
  \subseteq
  \Pi_{\Delta_{m+c}}((\sigma,x))
\]
for a constant $c$.

\begin{proof}[Proof of Theorem~\ref{thm:randomness-complexity}]
  Let us replace~\eqref{eq:randomness-complexity} by the equivalent
  \begin{equation*}
    \sigma\in\mathcal{D}_m
    \Longrightarrow
    \sigma\in\mathcal{U}_{\le c \left|\sigma\right| 2^{-m} m^2}.
  \end{equation*}
  Define a prediction system $\mathcal{L}=(L_1,L_2,\ldots)$ as, essentially, $\mathcal{D}_m$;
  formally:
  \begin{itemize}
  \item
    The law of nature $L_1$ contains only finite data sequences $\sigma\in\mathcal{D}_m$ of length at most $2^m$.
    This set if prefix-free by the definition of a randomness-type prediction system:
    indeed, \eqref{eq:randomness-condition} shows that, for any infinite data sequence $\omega$,
    at most $2^{-m}2^m=1$ element of $L_1$ is a prefix of $\omega$.
  \item
    The next 2 laws of nature ($L_2$ and $L_3$) contain only finite data sequences $\sigma\in\mathcal{D}_m$
    of length in the range $2^m+1$ to $2^{m+1}$,
    and we will define them similarly
    to the proof
    of Theorem~\ref{thm:K-M}.
    Enumerate, in a computable manner,
    all such data sequences $\sigma$
    ($\sigma\in\mathcal{D}_m$ and $\left|\sigma\right|\in[2^m+1,2^{m+1}]$);
    let $\sigma_1,\sigma_2,\ldots$ be such an enumeration.
    For $n=1,2,\ldots$:
    include $\sigma_n$ into the law of nature ($L_2$ or $L_3$) with the smallest index
    that does not already contain data sequences comparable with $\sigma_n$
    in the sense of the order $\sqsubseteq$
    (in particular, $\sigma_1\in L_2$).
    Two laws of nature ($L_2$ and $L_3$) are sufficient since, by \eqref{eq:randomness-condition},
    each infinite data sequence $\omega$ has at most $2^{-m}2^{m+1}=2$ elements of $\mathcal{D}_m$
    with length in the range $[2^m+1,2^{m+1}]$ (and even $[0,2^{m+1}]$)
    as its prefixes.
  \item
    The next 4 laws of nature ($L_4$ to $L_7$) contain only finite data sequences $\sigma\in\mathcal{D}_m$
    of length in the range $2^{m+1}+1$ to $2^{m+2}$.
    Enumerate, in a computable manner,
    all data sequences $\sigma\in\mathcal{D}_m$ whose length is in this range;
    let $\sigma_1,\sigma_2,\ldots$ be such an enumeration.
    For $n=1,2,\ldots$:
    include $\sigma_n$ into the law of nature ($L_4$ to $L_7$) with the smallest index
    that does not already contain data sequences comparable with $\sigma_n$ in the sense of the order $\sqsubseteq$.
    We will never run out of the available laws of nature ($L_4$ to $L_7$)
    by the definition of a randomness-type prediction system:
    see \eqref{eq:randomness-condition}.
  \item
    And so on.
  \end{itemize}
  Any data sequence $\sigma\in\mathcal{D}_m$
  whose length $l$ is in the range $2^{m+i-1}+1$ to $2^{m+i}$, $i\in\mathbb{N}$,
  will be included in one of the $2^{i}$ laws of nature $L_{2^{i}}$ to $L_{2^{i+1}-1}$,
  and so
  \[
    \sigma
    \in
    \mathcal{L}_{\le2^{i+1}-1}
    \subseteq
    \mathcal{L}_{\le2^{2-m}(l-1)-1}.
  \]
  In combination with Theorem~\ref{thm:precise}, we obtain
  \[
    \sigma
    \in
    \mathcal{U}_{\le c_1 2^{K(\mathcal{L})} 2^{-m}l}
  \]
  for a constant $c_1>0$.
  Therefore, our task reduces to checking that
  \[
    2^{K(m)}
    \le
    c_2 m^2
  \]
  for a constant $c_2>0$.
  Since $2^{-K(m)}$ is the universal semimeasure on the positive integers (see, e.g., \cite{Shen:2015}, Theorem~7.29),
  we even have
  \[
    2^{K(m)}
    \le
    c_3 m (\log m) (\log\log m) \cdots (\log\cdots\log m),
  \]
  where the product contains all factors that are greater than~1
  (see \cite{Rissanen:1983}, Appendix~A).
\end{proof}

\begin{remark}
  The proof shows that the inclusion~\eqref{eq:randomness-complexity} can be strengthened to
  \begin{equation*}
    \Pi^{c l 2^{K(m)-m}}_{\mathcal{U}}((\sigma,x))
    \subseteq
    \Pi_{\mathcal{D}_m}((\sigma,x)).
  \end{equation*}
\end{remark}

Next we show how the constant $c$ in Theorem~\ref{thm:universal-randomness-type} depends on $\Lambda$.
First we give a standard definition of prefix complexity adapted to randomness-type prediction systems.

A \emph{description language for randomness-type prediction systems} is a function $F$
mapping $\mathbb{2}^*$ to the set of all randomness-type prediction systems such that the set
\[
  \left\{
    (d,\sigma,m)\in\mathbb{2}^*\times(\mathbf{X}\times\mathbb{2})^*\times\mathbb{N}
    \st
    \sigma \in F_m(d)
  \right\}
\]
is recursively enumerable,
where $F_m(d)$ is the $m$th set in $F(d)$,
i.e., $F_m(d):=\Lambda_m$ when $F(d)=(\Lambda_1,\Lambda_2,\ldots)$.
The \emph{effective domain} $\dom(F)$ of a description language $F$ for randomness-type prediction systems
is \eqref{eq:dom}.
A \emph{prefix-free description language for randomness-type prediction systems}
is a description language $F$ for randomness-type prediction systems
such that $\dom(F)$ is prefix-free.

The \emph{complexity} of a randomness-type prediction system $\Lambda$
with respect to a description language $F$ for randomness-type prediction systems
is defined by
\[
  C_F(\Lambda)
  :=
  \min
  \left\{
    \left|d\right|
    \st
    \Lambda = F(d)
  \right\}.
\]
Analogously to Theorem~\ref{thm:complexity-of-prediction-systems}
(but using the fact that we can enforce item~\ref{it:1a} on p.~\pageref{it:1a})
we can prove:

\begin{theorem}
  There is a description language $U$ for randomness-type prediction systems that is universal
  in the sense that for any description language $F$ for randomness-type prediction systems
  there exists a constant $c$ such that,
  for any randomness-type prediction system $\Lambda$,
  \begin{equation*}
    C_U(\Lambda)
    \le
    C_F(\Lambda) + c.
  \end{equation*}
  There is a prefix-free description language $U'$ for randomness-type prediction systems that is universal
  in the sense that for any prefix-free description language $F$ for randomness-type prediction systems
  there exists a constant $c$ such that,
  for any randomness-type prediction system $\Lambda$,
  \[
    C_{U'}(\Lambda)
    \le
    C_F(\Lambda) + c.
  \]
\end{theorem}

We fix a universal description language $U$ for randomness-type prediction systems
and call $C(\Lambda):=C_U(\Lambda)$ the \emph{plain complexity} of $\Lambda$.
And we fix a universal prefix-free description language $U'$ for randomness-type prediction systems
and call $K(\Lambda):=C_{U'}(\Lambda)$ the \emph{prefix complexity} of $\Lambda$.

\begin{theorem}
  There exists a constant $c\in\mathbb{N}$ such that,
  for any randomness-type prediction system $\Lambda$ and any $m\in\mathbb{N}_0$,
  $\Lambda_{m+K(\Lambda)+c}\subseteq\mathcal{D}_m$.
\end{theorem}

\begin{proof}
  Let $U'$ be our chosen universal prefix-free description language
  for randomness-type prediction systems.
  Analogously to the proof of Theorem~\ref{thm:universal-randomness-type},
  we can then combine all randomness-type prediction systems into one system $\mathcal{D}'$
  by setting
  \begin{equation}\label{eq:D}
    \mathcal{D}'_m
    :=
    \bigcup_{d\in\mathbb{2}^*}
    U'_{m+\left|d\right|}(d)
  \end{equation}
  (cf.\ \eqref{eq:setting}).
  We again get a randomness-type prediction system:
  \begin{align*}
    \left|
      \left\{
        i\in\{1,\ldots,l\}
        \st
        \omega^i\in\mathcal{D}'_m
      \right\}
    \right|
    &=
    \left|
      \bigcup_{d\in\mathbb{2}^*}
      \left\{
        i\in\{1,\ldots,l\}
        \st
        \omega^i\in U'_{m+\left|d\right|}(d)
      \right\}
    \right|\\
    &\le
    \sum_{d\in\dom(U')}
    \left|
      \left\{
        i\in\{1,\ldots,l\}
        \st
        \omega^i\in U'_{m+\left|d\right|}(d)
      \right\}
    \right|\\
    &\le
    \sum_{d\in\dom(U')}
    2^{-m-\left|d\right|} l
    \le
    2^{-m} l.
  \end{align*}
  The inclusion $\Lambda_{m+K(\Lambda)}\subseteq\mathcal{D}'_m$
  now follows from $\Lambda=U'(d)$ for some $d\in\mathbb{2}^{K(\Lambda)}$.
  The addend ``${}+c$'' allows us to replace the randomness-type prediction system $\mathcal{D}'$ defined by~\eqref{eq:D}
  by our chosen universal randomness-type prediction system $\mathcal{D}$.
\end{proof}

In conclusion of this section we will reword our definition of a universal prediction system
to make it more similar to that of a universal randomness-type prediction system.
A \emph{complexity-type prediction system} is a jointly enumerable family $\Lambda$
of sets $\Lambda_m\subseteq(\mathbf{X}\times\mathbb{2})^*$ of finite data sequences such that,
for all $m\in\mathbb{N}_0$ and $\omega\in(\mathbf{X}\times\mathbb{2})^{\infty}$,
\begin{equation}\label{eq:requirement}
  \left|
    \left\{
      i\in\mathbb{N}
      \st
      \omega^i\in\Lambda_m
    \right\}
  \right|
  \le
  2^{m}.
\end{equation}

\begin{theorem}\label{thm:universal-complexity-type}
  There exists a complexity-type prediction system $\mathcal{V}$ such that,
  for any complexity-type prediction system $\Lambda$,
  there exists $c\in\mathbb{N}$ such that, for all $m\in\mathbb{N}_0$,
  $\Lambda_{m}\subseteq\mathcal{V}_{m+c}$.
\end{theorem}

Fix a complexity-type prediction system $\mathcal{V}$
satisfying the condition in Theorem~\ref{thm:universal-complexity-type}
and call it the \emph{universal complexity-type prediction system}.
The following analogue of Theorem~\ref{thm:connection-1} shows that this is not an essentially new notion.

\begin{theorem}\label{thm:connection-5}
  There is a constant $c>0$ such that, for all $m\in\mathbb{N}_0$,
  \begin{equation}\label{eq:connection-5}
    \{\sigma\st\CCC(\sigma)\le m-c\}
    \subseteq
    \mathcal{V}_{m}
    \subseteq
    \{\sigma\st\CCC(\sigma)\le m+c\}.
  \end{equation}
\end{theorem}

\begin{proof}
  The left-hand inclusion in~\eqref{eq:connection-5} is obvious.
  The right-hand inclusion is witnessed by the following description language for laws of nature.
  Enumerate in a computable manner all finite data sequences $\sigma_1,\sigma_2,\ldots$
  in $\mathcal{V}_m$.
  Order all binary strings in $\mathbb{2}^m$ lexicographically.
  Assign $0^m$ (i.e., the sequence $(0,\ldots,0)$ of length $m$)
  to $\sigma_1$ as its description.
  For $i=2,3,\ldots$, assign to $\sigma_i$ as its description
  the first string in $\mathbb{2}^m$ that has not being assigned as yet
  to the strings among $\sigma_1,\ldots,\sigma_{i-1}$ that are comparable with $\sigma_i$.
  The finite data sequences with the same description
  now form a law of nature with that description.
  Repeat for all $m\in\mathbb{N}_0$.

  The only thing that remains to be checked is that we will never run out of strings in $\mathbb{2}^m$.
  Let us check this carefully.
  It is convenient to think of the elements of $\mathbb{2}^m$ as colours,
  and our goal is to show that we will never run out of the $2^m$ available colours.
  We look at the set $(\mathbf{X}\times\mathbb{2})^*$ of all finite data sequences as a tree
  (rooted at $\Box$ and with $\sigma$ and $\sigma'$ connected with an edge
  when $\sigma\sqsubset\sigma'$ but there is no $\sigma''$ such that $\sigma\sqsubset\sigma''\sqsubset\sigma'$).
  \emph{Siblings} are non-empty finite data sequences that differ only in their last element.
  Let us fix some stage $i$ of the construction in the previous paragraph;
  at the beginning of this stage we have a partial colouring of the tree $(\mathbf{X}\times\mathbb{2})^*$:
  the vertices $\sigma_1,\ldots,\sigma_{i-1}$ have been coloured,
  and our task is to colour $\sigma_i$.
  For each vertex $\sigma$,
  let $T_{\sigma}$ be the set of all colours used in the tree rooted at $\sigma$,
  and $P_{\sigma}$ be the set of all colours used along the path from the root $\Box$ to $\sigma$
  (not including $\sigma$).
  Notice the following properties of our construction:
  \begin{enumerate}
  \item\label{it:1b}
    The colours of comparable vertices are different.
  \item
    If a vertex $\sigma$ gets colour $d$,
    then each smaller color is used either for a predecessor of $\sigma$ or for a descendant of $\sigma$.
  \item
    If $\sigma$ is a vertex (coloured or not),  
    then the sets $T_{\sigma}$ and $P_{\sigma}$ are disjoint (by Property~\ref{it:1b}),
    and $T_{\sigma}$ is an initial segment in the complement $\mathbb{2}^d\setminus P_{\sigma}$ of $P_{\sigma}$.
    Indeed, if a colour appears in $T_{\sigma}$,
    it is the colour of some vertex $\sigma'\sqsupseteq\sigma$,
    and so all smaller colours appear either before $\sigma'$
    (therefore, in $P_{\sigma}$ or $T_{\sigma}$)
    or after $\sigma'$ (therefore, in $T_{\sigma}$).
  \item
    The sets $T_{\sigma}$ and $T_{\sigma'}$ for any two siblings $\sigma$ and $\sigma'$ are comparable with respect to inclusion.
    Indeed, they are two initial segments of the same complement.
  \item
    For each vertex $\sigma$ the total number of colours used in $T_{\sigma}$ is minimal,
    in the sense of being equal to the maximal number of coloured vertices on the paths in $T_{\sigma}$.
    This can be shown by an inductive argument using the previous property.
  \end{enumerate}
  The last property, in combination with \eqref{eq:requirement},
  shows that we will have at least one colour left for $\sigma_i$.
\end{proof}

\section{Universal conformal prediction under the IID assumption}
\label{sec:conformal}

Up to this point our exposition has been completely probability-free,
but in this section we will consider the special case
where the data are generated in the IID manner.
For basic definitions of the theory of conformal prediction see, e.g., \cite{Vovk:2014Bala}.
For simplicity, we will only consider computable conformity measures
that take values in the set~$\mathbb{Q}$ of rational numbers.
Remember that $\mathcal{D}$ is the universal randomness-type prediction system,
as introduced in the previous section;
let us set $\mathcal{D}_m:=(\mathbf{X}\times\mathbb{2})^*$ for $m<0$
(i.e., we include all finite data sequences in $\mathcal{D}_m$ for negative $m$).

\begin{theorem}\label{thm:conformal}
  Let $\Gamma$ be a conformal predictor
  based on a computable conformity measure
  taking values in~$\mathbb{Q}$.
  Then there exists $c\in\mathbb{N}$ such that,
  for almost all infinite data sequences
  $\omega=((x_1,y_1),(x_2,y_2),\ldots)\in(\mathbf{X}\times\mathbb{2})^{\infty}$
  and all significance levels $\epsilon\in(0,1)$,
  from some $l$ on we will have
  \begin{equation}\label{eq:conformal}
    \Pi_{\mathcal{D}_{\lfloor-\log\epsilon\rfloor-c}}((\omega^{l-1},x_{l}))
    \subseteq
    \Gamma^{\epsilon}((\omega^{l-1},x_{l})).
  \end{equation}
\end{theorem}

This theorem says that the prediction set output by the universal randomness-type prediction system
is at least as precise as the prediction set output by $\Gamma$,
to within the usual additive constant.

\begin{proof}[Proof of Theorem~\ref{thm:conformal}]
  Without loss of generality we can and will assume $\epsilon\in(0,1/2)$.
  For each such $\epsilon$ set $m:=\lfloor-\log\epsilon\rfloor-1$.
  (Intuitively, we replace $\epsilon$ by a new significance level $2^{-m}$,
  which we make at least twice as large as the original $\epsilon$.)
  Let $\Lambda_m$ be $\Gamma^{2^{-m}}$ forced to satisfy~\eqref{eq:randomness-condition};
  formally, $\Lambda_m$ contains only finite data sequences $\sigma$
  such that $\Gamma^{2^{-m}}$ makes an error when predicting the last label in $\sigma$,
  and $\Lambda$ is defined by induction first on $m$ and then on the length of $\sigma$ as follows:
  $\sigma$ is included in $\Lambda_m$ if and only if:
  \begin{itemize}
  \item
    $\sigma$ is included in all $\Lambda_i$, $i<m$
    (this condition is satisfied automatically if $m=1$);
  \item
    the condition~\eqref{eq:randomness-condition} is satisfied,
    where $l$ is the length of $\sigma$ and $\omega$ is an infinite continuation of $\sigma$.
  \end{itemize}
  By the standard validity property of conformal predictors (\cite{Vovk:2014Bala}, Corollary~1.1),
  we will have
  \begin{equation*}
    \Pi_{\Lambda_{m}}((\omega^{l-1},x_{l}))
    \subseteq
    \Gamma^{\epsilon}((\omega^{l-1},x_{l}))
  \end{equation*}
  from some $l$ on almost surely.
\end{proof}

\begin{remark}
  The proof shows that we can replace the $c$ in \eqref{eq:conformal} by $c+K(\Gamma)$,
  where $c$ now does not depend on $\Gamma$
  and $K(\Gamma)$ is the smallest prefix complexity of the programs
  for computing the conformity measure on which $\Gamma$ is based.
\end{remark}

\section{The theory of Kolmogorov complexity}
\label{sec:Kolmogorov}

In this section we will discuss the theory of Kolmogorov complexity
as a special case of our theory.
We obtain the former by taking $\mathbf{X}$ and the label space
($\mathbb{2}$ in this paper) of size one.
More generally, the theory of Kolmogorov complexity embeds
into our theory when we fix an object and a label
and only consider sequences of identical observations
with those object and label.
Therefore, let us fix an element $\mathbf{x}$ of $\mathbf{X}$
and a label, say $0$.

Let $o\in(\mathbf{X}\times\mathbb{2})^*$
be the infinite data sequence $((\mathbf{x},0),(\mathbf{x},0),\ldots)$
consisting of identical observations $(\mathbf{x},0)$.

\begin{theorem}\label{thm:Kolmogorov}
  \begin{align}
    \CCC(o^n) &= C(n)+O(1),
    \label{eq:Kolmogorov-C}\\
    \KKK(o^n) &= K(n)+O(1),
    \notag\\
    -\log\MMM(o^n) &= -\log M(n)+O(1).
    \notag
  \end{align}
\end{theorem}

\begin{proof}
  We will only prove~\eqref{eq:Kolmogorov-C};
  the other two relations can be proved similarly.
  Reinterpreting a description of $n\in\mathbb{N}$
  as a description of the law of nature $(\mathbf{X}\times\mathbb{2})^{n}$,
  we obtain the inequality $\le$ in~\eqref{eq:Kolmogorov-C}.
  (Alternatively, we can notice that~\eqref{eq:Kolmogorov-C}
  is a special case of the inequality $\le$ of~\eqref{eq:C-trivial}.)
  And reinterpreting a description of a law of nature $L$
  as a description of the length of the only element of $L\cap\Sigma(o)$,
  we obtain the inequality $\ge$ in~\eqref{eq:Kolmogorov-C}.
\end{proof}

Combining Theorem~\ref{thm:Kolmogorov} with the standard fact that $K(n)=-\log M(n)+O(1)$
(e.g., \cite{Shen:2015}, Theorem 7.29),
we can see that Theorem~\ref{thm:K-M} can be improved when restricted to $\Sigma(o)$:
in this case $-\log\MMM=\KKK+O(1)$.
Unfortunately, the equality cannot be extended to all finite data sequences:
see Theorem~\ref{thm:Shen}.

\section{Conclusion}

In this paper we have ignored the computational resources,
first of all, the required computation time and space (memory).
Developing versions of our definitions and results taking into account the time of computations is a natural next step.
In analogy with the theory of Kolmogorov complexity,
we expect that the simplest and most elegant results will be obtained for computational models
that are more flexible than Turing machines,
such as Kolmogorov--Uspensky algorithms and Sch\"onhage machines.

An interesting open question is whether Theorem~\ref{thm:K-M} can be improved
to $-\log\MMM=\KKK+O(1)$ by modifying the definition of prefix time complexity
(Theorem~\ref{thm:Shen} says that a modification is necessary,
and Remark~\ref{rem:Andreev} shows that intermediate time complexity does not work).
Another open question is whether plain complexity $C$
can be improved to (or almost to) prefix complexity $K$ in Theorem~\ref{thm:precise-opposite}.

More open questions are raised by the definition of universal randomness-type prediction systems
in Section~\ref{sec:randomness}:
how can such prediction systems be characterized in terms of other notions
(such as plain and prefix time complexity, time randomness deficiency, and \emph{a priori} time semimeasure)
introduced in this paper or in terms of similar notions?
(In Theorem~\ref{thm:randomness-complexity} we gave only the most obvious connection.)

\subsection*{Acknowledgments}

We thank the anonymous referees of the conference and journal versions of this paper for helpful comments.
In particular, comments made by the referees of the journal version
have led to Remarks~\ref{rem:Shen-1} and~\ref{rem:Shen-2},
and we especially appreciate their generosity in filling a gap in the proof of Theorem~\ref{thm:connection-5}.
  This work has been supported by the Air Force Office of Scientific Research (grant ``Semantic Completions''),
  EPSRC (grant EP/K033344/1),
  and the EU Horizon 2020 Research and Innovation programme (grant 671555).

\end{document}